\documentclass[11pt, twoside, letter]{article}
\usepackage[margin = 1in]{geometry}
\usepackage{array}
\usepackage{titlesec}
\titlespacing*{\section}{0pt}{1.1\baselineskip}{\baselineskip}
\usepackage{subcaption}
\usepackage{amsthm}
\usepackage{amsmath}
\usepackage{amsfonts}
\usepackage{microtype}
\usepackage{graphicx}
\usepackage{comment}
\usepackage{booktabs} 
\usepackage[english]{babel}
\usepackage{amsthm}
\usepackage{bm}
\usepackage{algorithm}
\usepackage[colorlinks]{hyperref}
\usepackage{graphicx}
\usepackage{xcolor}
\usepackage{algorithmic}

\newtheorem{theorem}{Theorem}

\newtheorem{lemma}{Lemma}
\newtheorem{corollary}{Corollary}

\newtheorem*{definition2*}{Non-negative Matrix Factorization}
\newtheorem*{definition3*}{NMF under simplicial constraints}
\newtheorem*{definition4*}{Multiplicative Weights Update}

\newtheorem{example}{Example}

\newtheorem*{result}{Main Theorem}

\def\be{\begin{equation}}
\def\ee{\end{equation}}
\def\bal{\begin{aligned}}
\def\eal{\end{aligned}}
\def\bi{\begin{itemize}}
\def\ei{\end{itemize}}

\DeclareMathOperator*{\argmax}{arg\,max}

\allowdisplaybreaks

\newcommand{\R}{\mathbb{R}}

\newcommand{\norm}[2][]{\ensuremath{\left\Vert #2 \right\Vert{#1}}}

\renewcommand{\vec}[1]{\mathbf{#1}}

\begin{document}

\title{Convergence to Second-Order Stationarity for Non-negative Matrix Factorization: Provably and Concurrently}

\author{Ioannis Panageas\\SUTD\\ioannis@sutd.edu.sg
\and Stratis Skoulakis\\SUTD\\efstratios@sutd.edu.sg
\and Antonios Varvitsiotis\\SUTD\\antonios@sutd.edu.sg
\and Xiao Wang\\SUTD\\xiao\_wang@sutd.edu.sg
}

\date{}
\maketitle

\begin{abstract}
Non-negative matrix factorization (NMF) is a  fundamental non-convex optimization problem with numerous applications in Machine Learning (music analysis, document clustering, speech-source separation etc). Despite having received extensive study, it is poorly understood whether or not there exist natural algorithms that can provably converge to a local minimum. Part of the reason is because the objective is heavily symmetric and its gradient is not Lipschitz. In this paper we define a multiplicative weight update type dynamics (modification of the seminal Lee-Seung algorithm) that runs concurrently and provably avoids saddle points (first order stationary points that are not second order). Our techniques combine tools from dynamical systems such as stability and exploit the geometry of the NMF objective by reducing the standard NMF formulation over the  non-negative orthant to 
 a  new formulation over  a scaled  simplex. An important advantage of our method  is the use of concurrent updates, which  permits   implementations in parallel computing environments.

\end{abstract}

\section{Introduction}
Consider a non-negative data matrix $V\in \R_+^{n\times m} $ consisting  of  $m$ samples  each with $n$ non-negative features, arranged as its columns.
In the non-negative matrix factorization (NMF) problem  the goal  is to identify   two entrywise non-negative matrices $W\in \R_+^{n\times r}$ and $ H\in \R_+^{r\times m}$  (for some a-priori fixed $r\in\mathbb{N}$) such that the matrix product $WH$ approximates $V$,  where the precise sense of approximation    depends  on  the application of interest.

Factorizations based on  non-negative matrices  have found  numerous  applications throughout most   fields of science and engineering, notable examples including  music analysis \cite{fevotte},  document clustering \cite{XLG03},    speech-source separation  \cite{speech-sep} and  cancer-class identification  \cite{GC05}. 
 NMFs are useful as they lead to   additive and sparse representations of the input data.
Indeed, a factorization  $ V=WH$ can be  interpreted  as
$V(:,j)=WH(:,j),$ for all $  j=1,\ldots,m,$
i.e.,
 each sample (columns of $V$) can be represented  as non-negative combination of basis vectors (columns of~$W$). 

The NMF problem was  introduced in the engineering community in the seminal paper  \cite{PT94} and was subsequently    popularized further in   \cite{LS99}. As it turns out, non-negative factrorizations were also explored earlier in the combinatorial optimization community, as
 a tool to describe  the efficacy of linear programming for hard combinatorial problems~\cite{Y91}.

 In terms of complexity, the  problem of deciding the existence of an exact NMF is  NP-hard even for   $r={\rm rank}(V)$ \cite{V09}. 
\cite{CR93} gave the first  finite  algorithm for calculating the  non-negative rank over the real numbers  based on quantifier elimination arguments. Extending this technique,  \cite{arora} derived an $(mn)^{O(2^{r^2})} $ algorithm for exact NMF that was subsequently improved to $(2^rmn)^{O(r^2)}$~\cite{moitra}.
Additionally,  \cite{arora} identified sufficient conditions on $V$ under which the NMF problem can be solved in polynomial time.

There are several  practically efficient algorithms for calculating (approximate) NMFs. The starting point for the majority of existing  approaches is the non-convex~program
\begin{equation}\label{main}\tag{NMF}
\min_{W\in \R_+^{n\times r}, H\in \R_+^{r\times m}} F(W,H):=\| V-WH\|_F^2,
\end{equation}

for some user-specified parameter  $r\in \mathbb{N}$.
Nevertheless, in many applications,  it is useful  to use  information-theoretic divergences rather than the squared Euclidean loss function.
In view of the  non-convexity of  \eqref{main},  the theoretical analysis  of  iterative algorithms  for solving \eqref{main} typically  amounts to showing   that an accumulation  point of the  sequence of iterates $(W^t,H^t) $ generated by the algorithm  satisfies the first-order optimality  conditions given in   \eqref{def:KKT}.

Many  algorithms for NMF can be interpreted in a unified manner within the framework of block coordinate  descent (BCD), which is also known as the Gauss-Seidel method, e.g., see  the survey~\cite{survey} and references therein. In the BCD framework the goal is to minimize a smooth function over a domain that  can be decomposed as  the Cartesian product of closed convex sets. At each iteration of the BCD method, the function is
minimized with respect to a single block of variables while the rest of the blocks are kept  fixed.
Convergence  results in the BCD framework typically require   attainment and unicity  of the minimizer at each step \cite{bertsekas}, the unicity assumption being redundant in the  2-block case   \cite{GS}.
One   example is the  Alternating Non-negative Least Squares  (ANLS) method \cite{KP08}, which is a 2-block BCD with   updates  given by the solutions of $\min_{W\ge 0}\|V-WH\|^2, $ and $ \min_{H\ge 0}\|V-WH\|^2$.
A second   example is the  Hierarchical ALS (HALS) method \cite{HALS}, which has been  rediscovered several times, e.g.   as the rank-one residue
iteration~\cite{RRI} and  FastNMF in \cite{LZ09}. 
 The HALS method corresponds to  a $2r$-block BCD,  one for each column of the factors  $W$ and $H$.
 For additional details concerning algorithms for NMF and their convergence properties the reader is referred to the surveys \cite{survey,gillis} and references therein.

 Going beyond the BCD framework,  the multiplicative update (MU) rule introduced in~\cite{LS01}
  is by far the most popular algorithm for NMF. In the MU framework, the updates are given by
$$W_{t+1}=  W_t\circ {VH_t^\top \over W_tH_tH_t^\top} \quad H_{t+1}= H_t\circ {W_{t+1}^\top V\over W_{t+1}^\top W_{t+1}H_t},$$
where $X/Y$ denotes  the componentwise  division of two matrices. \cite{LS01}  showed using a majorization-minimization  approach, e.g., see \cite{SBP},
 that under the MU  rule, the objective
function $\| V-WH\|_F^2$ is non-increasing. However, it has  long been observed  that the MU rule  may fail  to converge to  a first-order stationary point, e.g. see  \cite{GZ05}.
Indeed, only when  the MU method converges to a fixed point  $(W,H)$  with strictly positive entries
this is also  a KKT point.  An important limitation of  the MU update rule is that it cannot be implemented concurrently (as $H_{t+1}$ depends on $W_{t+1}$).  \cite{Lin07} further elaborates on  the  difficulties in proving convergence of the  MU algorithm, while establishing convergence  of a close variant of Lee-Seung's~method.

Despite the wealth of existing algorithmic approaches  for finding NMFs, a common limitation is that they may fail to  avoid  saddle points, e.g. see \cite{FS04,GZ05,L07,CDPR04,BBLPP06}. 
This  drawback    is also shared with many  gradient-based approaches applied to  important  optimization problems  (e.g. training neural networks, matrix completion, community detection), as such methods   are only guaranteed to  converge to first-order stationary points, among which  there exists a  proliferation of highly-suboptimal saddle points,
e.g., see \cite{bengio}.  At the same time, algorithmic approaches that incorporate additional curvature information
typically converge to  second-order stationary points, i.e., points with vanishing gradient and  positive semidefinite Hessian,  which turn out to be  are as good as local minima in many problems of practical interest \cite{nospurious}.
These observations have led to a flurry of research activity
on  avoiding
 saddle points the last 5 years (see \cite{Ge15, LP19, Jin17, ICMLPPW19, PPW19, Leearxiv} and references therein) for both unconstrained and constrained optimization. One other line of work that deals with avoiding saddle points in non-convex settings with linear constraints and can be applied to our paper (as long as one has shown our main Theorem \ref{t:main2}) can be found in \cite{Meisam1, Meisam2}.
 Our main goal in  this work,
  is to  identify new  methods that converge to second-order stationary points of \eqref{main}. 

\subsection{Summary  of  results and significance }
Before we formally state
our results, we provide some standard definitions in constrained optimization applied to the NMF problem (\ref{main}). Moreover we give the definition of our Multiplicative Weights Update (MWU) applied to NMF.

\paragraph {Stationary points of \eqref{main}.}
A pair of matrices $(W^*, H^*)$ is  a first-order stationary point  (FOSP) of problem (\ref{main}) if for all $i \in [n], k \in [r],  j \in [m]$ it satisfies:
\begin{equation}\label{def:KKT}
\begin{aligned}
    & W_{ik}^*, H_{kj}^* \geq 0,\\
    & W_{ik}^*>0  \implies
\frac{\partial F(W^*,H^*)}{\partial W_{ik}} = 0, \\
&  H_{kj}^*>0 \implies
 \frac{\partial F(W^*,H^*)}{\partial H_{kj}} = 0, \\
&  W_{ik}^*=0 \implies
\frac{\partial F(W^*,H^*)}{\partial W_{ik}} \geq  0, \\
&  H_{kj}^*=0 \implies
\frac{\partial F(W^*,H^*)}{\partial H_{kj}} \geq 0.
\end{aligned}
\end{equation}
Additionally,
$(W^*, H^*)$ is  a second-order stationary point ({SOSP}) of  the problem (\ref{main}) if it is a  {FOSP}, 
and moreover,
  \begin{equation}\label{def:secondorder}
  (\textrm{vec}(W) ^{\top}, \textrm{vec}(H) ^{\top}) \nabla^2 F(W^*,H^*) \left(\begin{array}{cc}\textrm{vec}(W)\\\textrm{vec}(H)\end{array}\right)\geq 0,
  \end{equation}
for any pair of matrices $(W,H)$ satisfying: 
 \begin{equation*}
    \bal
     &\frac{\partial F(W^*,H^*)}{\partial W_{ik}}>0 \implies W_{ik} = 0, \;\;\;W^*_{ik}=0 \implies W_{ik} \geq 0,\\
     & \frac{\partial F(W^*,H^*)}{\partial H_{kj}}>0 \implies  H_{kj} = 0, \;\;\; H^*_{kj}=0 \implies H_{kj} \geq 0,
\eal
\end{equation*}
for all $i \in [n], k \in [r], j \in [m].$

We now  present our \textit{multiplicative weight update} method, which  converges to a
SOSP of \eqref{main}
with probability $1$.

\begin{algorithm}[h]
  \caption{Concurrent Multiplicative Weight Update}

  \textbf{Input:} A matrix $V_{n \times m}$ with positive elements\\
  \textbf{Output:} Matrices $W_{n \times r}, H_{r \times m}$ s.t. $V \simeq W\cdot H$.

  \begin{algorithmic}

  \STATE $C= 4 r (nm)^{1/4} \sqrt{\norm{V}_F}, \ \  \epsilon = \Theta(\frac{1}{C^2})$

  \STATE $(W_0,H_0) \leftarrow$ a random point in simplex

  \STATE $V\leftarrow V/C^2$

  \FOR{$t=1$ to $T$}

    \STATE  $W_{ik}^{t+1} = W_{ik}^t \cdot (1 - \epsilon \cdot  \frac{\partial F(W^t,H^t)}{\partial W_{ik}})/Z$

    \STATE  $H_{kj}^{t+1} = H_{kj}^t (1 - \epsilon \cdot  \frac{\partial F(W^t,H^t)}{\partial H_{kj}})/Z$

    \STATE $Z = 1 - \epsilon \cdot \left(\sum_{i,k}W_{ik}^t \cdot  \frac{\partial F(W^t,H^t)}{\partial W_{ik}}+\sum_{k,j}H_{kj}^t \cdot  \frac{\partial F(W^t,H^t)}{\partial H_{kj}}\right)$
  \ENDFOR
  \STATE \textbf{return } $W \leftarrow  C \cdot W^{T+1}, H \leftarrow C \cdot H^{T+1}$.
  \end{algorithmic}
  \label{alg:MWU}
\end{algorithm}

Denote by $S$ the set of fixed points of MWU dynamics/algorithm (\ref{alg:MWU}), namely the set of points $(W,H)$ that are invariant under the update rule of MWU, multiplied by scalar $C$.  Then, the set of stationary points of problem (\ref{main}) with the additional constraint that $\sum_{i,j}W_{ij}+\sum_{i,j}H_{ij} = C$ is a subset of $S$.

We  now state the main result of our paper, which  informally  states
that the MWU dynamics (\ref{alg:MWU}) provably avoids fixed points that are not second order stationary points.
\begin{result}\label{thm:mwuaconverges}
The MWU described in Algorithm~\ref{alg:MWU} converges to the set of fixed points for any initialization $(W^0,H^0) \in \Delta_{nr+rm}$. Moreover, the set of initial conditions so that MWU converges to a point $(\tilde{W},\tilde{H})$ for which $(W^*,H^*):=C \cdot (\tilde{W},\tilde{H})$ is not a second order stationary point for problem (\ref{main}) is of measure zero (in $\Delta_{nr+rm}$).
\end{result}

An immediate corollary is the following:
\begin{corollary}
Assume that the iterate $(W^t,H^t)$ converges to a limit, under random initialization (any probability distribution that is absolutely continuous with respect to Lebesgue measure on $\Delta_{nr+rm}$ suffices for the initialization) the probability of MWU (\ref{alg:MWU}) to converge to a point $(\tilde{W},\tilde{H})$ so that $(W^*,H^*):=C \cdot (\tilde{W},\tilde{H})$ is a second order stationary point for problem (\ref{main}) is one.
\end{corollary}

An important differentiation between  the concurrent MWU rule  
and
existing  MWUs and
gradient based approaches, is the concurrent way of updating the entries
of $W$ and $H$. Both  iterates  $W^{t+1}$ and $H^{t+1}$ are updated
 using only the values of $W^{t}$ and $H^{t}$, an extremely useful   algorithmic feature
since it permits   implementations in parallel computing environments.
To the best of our knowledge, Algorithm~\ref{alg:MWU} is the first iterative method
for NMF that converges while performing its updates in
a concurrent way. In all  previous gradient based approaches
(both in multiplicative weight algorithms and in alternating least squares)
the convergence properties heavily rely on the fact that
the entries of $W$ and $H$ are updating in an alternating way,
while their concurrent counterparts may fail to converge. Such an instance is presented
in Example~\ref{e:1} for the Lee-Seung algorithm.

\subsection{Proof  techniques}
The main challenge in the non-convex problem of NMF is that the landscape is not Lipschitz (in the positive orthant) and there are continuums of stationary points (if $W,H$ is a stationary point, so it is $WD,D^{-1}H$ with $D$ a diagonal matrix with positive entries). The first challenge is essentially circumvented by adding an extra linear constraint that makes the feasibility region a compact set. In particular we define a modification of the NMF problem with the extra constraint that the sum of the entries of $W$ and $H$ is equal to a specific constant $C$ sufficiently large (this constant is chosen to be $C>2\cdot r\cdot (nm)^{1/4} \sqrt{\norm{V}_F}$). By choosing this constant that large, we are able to prove that all the fixed points of MWU that are not second order stationary points of the NMF problem are repelling for MWU inside simplex. Moreover, we are able to show that for any stationary point $(W^*,H^*)$ of NMF, there exists another stationary point $\tilde{W},\tilde{H}$ so that $W^*H^* = \tilde{W}\tilde{H}$ and the sum of entries of $\tilde{W}$ and $\tilde{H}$ is exactly $C$. As long as we have shown these claims, we use a result from \cite{ICMLPPW19} (Theorem \ref{thm:panageas}) that states that given any twice differentiable polynomial function $f(x)$ that we want to maximize with simplex constraints, MWU dynamics converges to second order stationary points almost surely. Last but not least, we can show that after adding the simplicial constraint (that is the sum of entries of $W,H$ must be equal to $C$) for any stationary point $(W^*,H^*)$ of the classic NMF problem there exists a stationary point $(\tilde{W},\tilde{H})$ such that $F(W^*,H^*) = F(\tilde{W},\tilde{H})$ (they have same values) and moreover $(\tilde{W},\tilde{H})$ lies in the positive orthant and satisfies the aforementioned simplicial constraint (see Lemma \ref{lem:equivalent}).

\textbf{Notation.} We use $W^i$ to denote the $i$-th column of matrix $W$ and $H_i$ to denote the $i$-th row of matrix $H$. We also use subscripts or superscripts with letter $t$ to denote the $t$-th iterate. We denote by $\textrm{vec}(A)$ the standard vectorization of matrix $A$, by $[n]$ the set $\{1,2,...,n\}$ and by $\Delta_n$ the simplex of size $n$, that is $\Delta_n =\{x \in \mathbb{R}^n: x\geq 0, \sum_{i=1}^n x_i=1\}.$
\section{Non-negative matrix factorization under simplicial constraints}
Before giving the details of our proof, we elaborate the strong relation between {Algorithm 1}
and the following optimization problem, which we call simplex-NMF:
\begin{equation}\label{main2}\tag{S-NMF}
\begin{aligned} \min  \ & \norm{ V-WH}_F^2,\\
{\rm s.t. } & \sum\limits_{i,k}W_{ik} + \sum\limits_{k,j}H_{kj} = C\\
&\  W\in \R_+^{n\times r}, H\in \R_+^{r\times m},
\end{aligned}
\end{equation}
where $C$ is any constant  $>2 r (nm)^{1/4} \sqrt{\norm{V}_F}$.

Problem  \eqref{main2} is  similar to  the original NMF problem,  the only difference being the additional simplex constraint.
 On the negative side, the feasibility set of \eqref{main2} is a strict subset of the feasibility set
of the original NMF problem, meaning that it may include solutions with cost much greater than the optimal value of the original NMF problem. On the positive side, this problem turns out to be \textit{algorithmically easier} to
tackle. More precisely, due to the recent result \cite{ICMLPPW19},  when   $C>2 r (nm)^{1/4} \sqrt{\norm{V}_F}$  the  sequence of matrices generated  by the MWU algorithm will converge almost surely to a SOSP of \eqref{main2}.
\paragraph {Stationary points of \eqref{main2}.}
A pair of matrices $(W^*, H^*)$ is  a first-order stationary point of
the problem
(\ref{main2}) if
for all $i \in [n], j\in [m], k \in [r],$ we have:

  \begin{equation}\label{def:KKT2}
  \begin{aligned}
    &  W_{ik}^*, H_{kj}^* \geq 0, \\
          & \sum_{i,k} W^*_{ik} + \sum_{k,j} H^*_{kj} = C,\\
&  W_{ik}^*>0 \implies  \frac{\partial F(W^*,H^*)}{\partial W_{ij}} = c, \text { for } i \in [n], k \in [r]\\
&  H_{kj}^*>0 \implies \frac{\partial F(W^*,H^*)}{\partial H_{kj}} = c,\text { for }  k \in [r], j \in [m]\\
\end{aligned}
\end{equation}
  \begin{equation*}
  \begin{aligned}
&  W_{ik}^*=0 \implies  \frac{\partial F(W^*,H^*)}{\partial W_{ik}} \geq c,  \text { for } i \in [n], k \in [r]\\
&   H_{kj}^*=0 \implies  \frac{\partial F(W^*,H^*)}{\partial H_{kj}} \geq c, \text { for }  k \in [r], j \in [m]\\ 
\end{aligned}
\end{equation*}
for some constant $c$.
Additionally, a
pair of matrices $(W^*, H^*)$ is  a second-order stationary of the  problem
(\ref{main2})  if it  is a FOSP, 
 and moreover,
 \begin{equation}\label{def:secondordermodified}
 (\textrm{vec}(W) ^{\top}, \textrm{vec}(H) ^{\top}) \nabla^2 F(W^*,H^*) \left(\begin{array}{cc}\textrm{vec}(W)\\\textrm{vec}(H)\end{array}\right)\geq 0,
 \end{equation}
 for any pair of  matrices $(W,H)$ satisfying:
%
%
\begin{equation*}
        \bal
         &  \frac{\partial F(W^*,H^*)}{\partial W_{ik}}>c \implies W_{ik} = 0, \;\;\;W^*_{ik}=0 \implies W_{ik} \geq 0,\\
 & \frac{\partial F(W^*,H^*)}{\partial H_{kj}}> c \implies H_{kj} = 0, \;\;\; H^*_{kj}=0 \implies H_{kj} \geq 0,\\
& \sum_{i=1}^n \sum_{k=1}^r W_{ik}  +
        \sum_{j=1}^m \sum_{k=1}^r H_{kj}
        =0.
        \eal
\end{equation*}


\begin{theorem}[\cite{ICMLPPW19}]\label{thm:panageas}
Consider the problem $\max\{ Q(x): x\ \in \Delta_d\}$ where  $Q: \mathbb{R}^d \to \mathbb{R}$ is a polynomial function. Then, the MWU algorithm with
update rule \begin{equation}\label{eq:easy}
x_{i}^{t+1} = x_{i}^t \frac{1 + \epsilon \frac{\partial Q}{\partial x_i}}{1 + \epsilon \sum_{j}\frac{\partial Q}{\partial x_j}}\  \textrm{ for all }i\in[d],
\end{equation}
 has the property that $Q(x^{t+1}) > Q(x^t)$ unless $x^t$ is a fixed point of the MWU dynamics   (\ref{eq:easy}). Moreover,
  the set of initial conditions so that the MWU dynamics (\ref{eq:easy}) converge to a point that is not a second order stationary point for $\max\{ Q(x): x\ \in \Delta_d\}$  is of measure zero. The statement holds when $\epsilon$ is chosen to be of order $\Theta(\frac{1}{L})$ where $L$ is the Lipschitz constant inside $\Delta_d$, that is $L:= \max  \{ \norm{\nabla Q(x)}_2: x \in \Delta_d\}$.
\end{theorem}

Notice that the MWU described in Algorithm~\ref{alg:MWU} is the same with the MWU of Theorem~\ref{thm:panageas} applied for $Q := -F$
(make it a minimization problem).
Also observe that, $F$ can be described as multivariate polynomial and thus Theorem~\ref{thm:panageas} applies if the parameter $\epsilon$ is selected appropriately small. For our case, $\epsilon$ should be $\Theta(1/C^2)$ (see also Algorithm~\ref{alg:MWU}).

To this end, Theorem~\ref{thm:panageas} ensures
that the MWU algorithm (almost certainly) converges to a SOSP of \eqref{main2}.
 However there is no  reason  why one should be interested in finding such points, since these points are not  necessarily SOSPs of \eqref{main}. 
   In fact, a SOSP of \eqref{main2}
    can be an arbitrarily bad solution for the initial NMF problem (e.g. consider the case $C = 0$). One of our main technical contributions consists in showing that if the offset $C$ exceeds a certain threshold (depending on $n,m,r,\norm{V}_F$) then the set of 
    SOSPs of \eqref{main2} is a subset of the second-order stationary points of \eqref{main}. Specifically, we show that:

\begin{theorem}\label{t:main2}
For any $C>2r(nm)^\frac{1}{4}\sqrt{\norm{V}_{F}}$ we have that
\bi
\item The value of \eqref{main2} is equal to \eqref{main}.
\item Any second-order stationary point of \eqref{main2} is also a second-order stationary point of \eqref{main}.
\ei
\end{theorem}
The proof of the  main theorem  follows easily by combining Theorem~\ref{t:main2}   with Theorem~\ref{thm:panageas}.
 The first part of Theorem~\ref{t:main2} is a consequence of   the following result, proven  in Section~\ref{sec:lem1}.
  \begin{lemma}\label{lem:equivalent}
For any first-order stationary point  of \eqref{main}  
 there exists a first-order stationary point of \eqref{main2} with the same value.
\end{lemma}
%
The second part of Theorem \ref{t:main2} heavily relies on the following result, which is proven in Section \ref{s:proof_lemma_1}.
\begin{lemma}\label{l:main1}
Any second-order stationary point   
of  the problem  \eqref{main2} necessarily satisfies  $c=0$. In particular, any  second-order stationary point   
of    \eqref{main2} is a first-order stationary point of \eqref{main}.
\end{lemma}
Lastly, note  that Lemma~\ref{l:main1} combined  with Theorem~\ref{thm:panageas} imply the following  claim:
\textit{If  we apply  the MWU to problem \eqref{main2} with $C > 2 r (nm)^{\frac{1}{4}}\sqrt{\norm{V}_{F}}$, the  generated  sequence of matrices will converge (almost certainly) to a pair of matrices $(W^*,H^*)$ that is a  FOSP of the problem \eqref{main}. }
Although the latter claim is not enough for our initial goal (finding pair of matrices that are second-order stationary points for NMF \eqref{def:secondorder}), it is the basic step of the proof of Theorem~\ref{t:main2}, that is presented in Section~\ref{s:proof_main_theorem}.

In Section~\ref{s:experiments}, we present the results of several experimental evaluations indicating that the MWU defined in Algorithm~\ref{alg:MWU} converges to the optimal pair of matrices. 

\subsection{An {illustrative} 
example}
Before proceeding we exhibit the above discussion in a very simple but illustrative example. Consider the following instance of the NMF problem:
\begin{equation*}
\min \{  (1 - x y)^2:   x,y \geq 0\}.
\end{equation*}
\noindent For the above optimization problem the
set of first-order stationary points (Equation \eqref{def:KKT}) is the union of the sets,
$\{(x,0): x \geq 0\}$,$\{(0,y): y \geq 0\}$, $\{(x,y): x\cdot y = 1, x \geq 0 , y \geq 0\}$. While the
set of the second-order stationary points (Equation~\eqref{def:secondorder}) is just the set $\{(x,y): x\cdot y = 1, x \geq 0 , y \geq 0\}$.
This simple example indicates the interest in finding second-order stationary points (Equation~\eqref{def:secondorder})
since the set of first order stationary points (Equation~\eqref{def:KKT}) contains very bad solutions.

Now consider the same problem with an additional simplicial constraint:
\begin{equation*}
\min \{ (1 - x y)^2:  \  x + y = 1,    \ x,y \geq 0\}.
\end{equation*}
\noindent In this case the
set of first-order stationary points (Equation~\eqref{def:KKT2}) is $\{(0,1) , (1,0) , (1/2 , 1/2)\}$. While the set of second order stationary points (Equation~\eqref{def:secondordermodified}) is the
set $\{(1/2 , 1/2)\}$. The above means that if we run the MWU algorithm runs with parameter $C=1$ then for almost all initializations, the produced sequence of solutions will converge to $(1/2 , 1/2)$, since this is the only second-order stationary point of the above minimization problem. Now notice that $(1/2,1/2)$ is not a good solution (for the initial optimization problem), the value of the global optimal (for the initial optimization problem) is $0$. More importantly, the point $(1/2 , 1/2)$ does not satisfy Equation~\eqref{def:secondorder}, which was our initial algorithmic goal. The reason for this is that the parameter $C$ is not chosen large enough (notice that $C$ must be selected as $2r(nm)^\frac{1}{4}\sqrt{\norm{V}_F})$ which is greater than $1$) and thus Theorem~\ref{t:main2} does not apply.

Now consider the problem with the same additional simplicial constraint, but with $C=4$.
\[\min \{ (1 - x y)^2: \  x + y = 4, \  x,y \geq 0 \}.\]
In this case the
set of first-order stationary points (Equation~\eqref{def:KKT2}) is $\{(0,1) , (1,0) , (2 , 2) , (2-\sqrt{3},2+\sqrt{3}),(2+\sqrt{3},2-\sqrt{3})\}$. While the set of second order stationary points (Equation~\eqref{def:secondordermodified}) is $\{(2-\sqrt{3},2+\sqrt{3}),(2+\sqrt{3},2-\sqrt{3})\}$. As a result, MWU algorithm with parameter $C=4$
converge either to $(2-\sqrt{3},2+\sqrt{3})$
or to $(2+\sqrt{3},2-\sqrt{3})$ for almost all initializations. Notice that both $(2-\sqrt{3},2+\sqrt{3})$ and
$(2+\sqrt{3},2-\sqrt{3})$ satisfy Equation~\eqref{def:secondorder} (in fact they are optimal solutions). This should not be a surprise since for $C=4$, Theorem~\ref{t:main2} applies and thus
MWU converges to second-order stationary points of Equation~\eqref{def:secondorder} .

\subsection{Calculating  derivatives}

The entries of the gradient of $F$ 
are given by:
\begin{equation}\label{col:gradients}
\begin{aligned}
 \frac{\partial F}{\partial W_{ik}} & =
    -2 \sum_{j=1}^m \left (V_{ij} - \sum_{\ell = 1}^r W_{i \ell} H_{\ell j } \right)H_{kj},
    \\
     \frac{\partial F}{\partial H_{kj}} & =
    -2 \sum_{i=1}^n \left (V_{ij} - \sum_{\ell = 1}^r W_{i \ell} H_{\ell j } \right)W_{ik},
\end{aligned}
\end{equation}
whereas  the entries of its  Hessian are  given by:
\begin{equation}\label{col:hessian}
\begin{aligned}
& \frac{\partial^2 F}{\partial^2 W_{ik}}  = 2\sum_{j=1}^m H_{kj}^2 \  \text{ and }  \ \frac{\partial^2 F}{\partial^2 H_{kj}} = 2\sum_{i=1}^n W_{ik}^2\\
&  \frac{\partial^2 F}{\partial W_{ik} \partial W_{i'\ell}}  = 0 \ \text{ and } \  \frac{\partial^2 F}{\partial H_{kj} \partial H_{\ell j'}} = 0 \textrm{ for }i\neq i', j\neq j'\\
 & \frac{\partial^2 F}{\partial W_{ik} \partial W_{i \ell}}  = 2\sum_{j=1}^m H_{\ell j} H_{kj}\\
& \frac{\partial^2 F}{\partial H_{kj} \partial H_{\ell j}}  = 2\sum_{i=1}^n W_{i\ell}W_{i k}\\
& \frac{\partial^2 F}{\partial W_{ik} \partial H_{kj}}  = -2V_{ij} + 2\sum_{\ell =1}^rW_{i\ell} H_{\ell j} + 2W_{ik}H_{kj}\\
& \frac{\partial^2 F}{\partial W_{ik} \partial H_{\ell j}}  = 2W_{i\ell} H_{kj}.
\end{aligned}
\end{equation}

\noindent Using
~\eqref{col:gradients} and~\eqref{col:hessian}
we arrive at  the following useful result: 

\begin{lemma}\label{l:main_technical}
For any pair of matrices $(W,H)$ and index $k \in \{1,\ldots,r\}$, let
the pair of matrices $(\tilde{W},\tilde{H})$ such that \textbf{1)} $\tilde{W}_{ik} = W_{ik}$ and $\tilde{W}_{i\ell} = 0$, \textbf{2)}
$\tilde{H}_{kj} = -H_{kj}$ and $\tilde{H}_{\ell j} = 0$. Then
\begin{equation*}\label{eq:KKTsecondorder2}
({\rm vec}(\tilde{W}) ^{\top}, {\rm vec}(\tilde{H}) ^{\top}) \cdot \nabla^2 F(W,H)=
\left(
\begin{array}{c}
\vec{0}^{\top} \\ -\nabla_{W^k}F \\ \vec{0}^{\top} \\ \nabla_{H_k}F \\ \vec{0}^{\top}
\end{array}
\right),
\end{equation*}
where $\vec{0}$ denotes the zero column vector of appropriate size.
\end{lemma}

\begin{proof}
Let us start by proving that for all the entries
of the vector corresponding respectively to $W_{ik}$ and $H_{kj}$,

\begin{itemize}
    \item $\left [({\rm vec}(\tilde{W}) ^{\top}, {\rm vec}(\tilde{H}) ^{\top}) \cdot \nabla^2 F(W,H)
\right]_{W_{ik}}
= -\frac{\partial F}{\partial W_{ik}}$

\item $\left [({\rm vec}(\tilde{W}) ^{\top}, {\rm vec}(\tilde{H}) ^{\top}) \cdot \nabla^2 F(W,H)
\right]_{H_{kj}}
= \frac{\partial F}{\partial H_{kj}}$
\end{itemize}

\noindent By direct calculation we get that,
\begin{eqnarray*}
\left [({\rm vec}(\tilde{W}) ^{\top}, {\rm vec}(\tilde{H}) ^{\top}) \cdot \nabla^2 F(W,H)
\right]_{W_{ik}}
&=& \frac{\partial^2 F}{\partial^2 W_{ik}} \tilde{W_{ik}} + 2\sum_{\ell \neq k}\frac{\partial^2 F}{\partial W_{ik} \partial W_{i \ell}}\tilde{W_{i\ell}}\\
&+& 2\sum_{j=1}^m \frac{\partial^2 F}{\partial W_{ik} \partial H_{kj}}\tilde{H}_{kj} +
2\sum_{j=1}^m\sum_{\ell \neq k} \frac{\partial^2 F}{\partial W_{ik} \partial H_{\ell j}}\tilde{H}_{\ell j}\\
&=& \frac{\partial^2 F}{\partial^2 W_{ik}} \tilde{W}_{ik} + 2\sum_{j=1}^m \frac{\partial^2 F}{\partial W_{ik} \partial H_{kj}}\tilde{H}_{kj}\\
&=& \frac{\partial^2 F}{\partial^2 W_{ik}} W_{ik} + 2\sum_{j=1}^m \frac{\partial^2 F}{\partial W_{ik} \partial H_{kj}}(-H_{kj})\\
&=& 2\sum_{j=1}^m W_{ik}H_{kj}^2
+ 2\sum_{j=1}^m \left[V_{i j} - \sum_{\ell=1}^r W_{i\ell} H_{\ell j}\right ] H_{k j}
- 2\sum_{j=1}^m W_{ik}H_{kj}^2\\
&=& -\frac{\partial F}{\partial W_{ik}}
\end{eqnarray*}
where the last equality follows by~(\ref{col:gradients}). Respectively for~$H_{kj}$.
Up next we prove that
\[\left [({\rm vec}(\tilde{W}) ^{\top}, {\rm vec}(\tilde{H}) ^{\top}) \cdot \nabla^2 F(W,H)
\right]_{W_{ik'}} = \left [({\rm vec}(\tilde{W}) ^{\top}, {\rm vec}(\tilde{H}) ^{\top}) \cdot \nabla^2 F(W,H)
\right]_{H_{k'j}}
= 0\]

\begin{eqnarray*}
\left [({\rm vec}(\tilde{W}) ^{\top}, {\rm vec}(\tilde{H}) ^{\top}) \cdot \nabla^2 F(W,H)
\right]_{W_{ik'}}
&=& \frac{\partial^2 F}{\partial^2 W_{ik'}} \tilde{W}_{ik'} + 2\sum_{\ell \neq k'}\frac{\partial^2 F}{\partial W_{ik'} \partial W_{i \ell}}\tilde{W}_{i\ell}\\
&+& 2\sum_{j=1}^m \frac{\partial^2 F}{\partial W_{ik'} \partial H_{k'j}}\tilde{H}_{k'j} +
2\sum_{j=1}^m\sum_{\ell \neq k'} \frac{\partial^2 F}{\partial W_{ik'} \partial H_{\ell j}}\tilde{H}_{\ell j}\\
&=& 2\frac{\partial^2 F}{\partial W_{ik'} \partial W_{i k}}\tilde{W}_{ik}
+ 2\sum_{j=1}^m \frac{\partial^2 F}{\partial W_{ik'} \partial H_{k j}}\tilde{H}_{k j}\\
&=& 2\frac{\partial^2 F}{\partial W_{ik'} \partial W_{i k}}W_{ik}
+ 2\sum_{j=1}^m \frac{\partial^2 F}{\partial W_{ik'} \partial H_{k j}}(-H_{k j})\\
&=&2\sum_{j=1}^m H_{k' j}H_{kj}W_{ik}
-2\sum_{j=1}^m W_{ik}H_{k'j}H_{kj} = 0.
\end{eqnarray*}
\noindent Respectively for $H_{k'j}$.
\end{proof}

\section{Poof of Lemma \ref{lem:equivalent}}\label{sec:lem1}
By Theorem 6, \cite{HD}, 
any FOSP  $(W^*,H^*)$ of  the problem \eqref{main} 
satisfies
$\norm{W^*H^*}_F \leq \norm{V}_F$. Consider the pair matrices $(\hat{W}, \hat{H})$  defined  by:
\[\hat{W}^k = \sqrt{\frac{\norm{H^*_k}_1}{\norm{W^{*k}}_1}} W^{*k} \ \text{  and  }  \ \hat{H}_k = \sqrt{\frac{\norm{W^{*k}}_1}{\norm{H^*_k}_1}} H^*_k,\]
 for each $k \in [r]$. Without loss of generality we assumed that both $\norm{W^{*k}}_1$ and $\norm{H^*_k}_1$ are not equal to zero (if one of these terms is zero, we may assume and the other is also zero and the inequality below still holds).
 By definition of $(\hat{W}, \hat{H})$  we have  that
 \[\hat{W} \hat{H} = \sum_{k=1}^r \hat{W}^k  \hat{H}_k = \sum_{k=1}^r W^{*k}  H^*_k = W^*  H^*,\]
  and thus $F(W^*,H^*) = F(\hat{W},\hat{H})$, i.e., these two pairs of matrices have the same  value.
Furthermore,  note that \begin{align*}
\norm{\hat{W}^k}_{1} + \norm{\hat{H}_k}_1 &= 2 \left({\norm{\hat{W}^{k}}_1 \norm{\hat{H}_k}_1}\right)^{{1\over 2}}= 2\left(\sum_{i,j} W^*_{ik} H^*_{k j}\right )^\frac{1}{2} \\&= 2\left [(\sum_{i,j} W^*_{ik} H^*_{k j})^2 \right]^\frac{1}{4}\\
&\leq 2\left[n m \sum_{i,j} W^{*2}_{ik} H^{*2}_{k j} \right]^\frac{1}{4}
\\
&{\le}  2(nm)^{\frac{1}{4}} \norm{{W^*H^*}}_{F}^\frac{1}{2} \leq  2(nm)^{\frac{1}{4}} \norm{V}_{F}^\frac{1}{2}=C.
\end{align*}
Lastly, consider the parametrized family  of matrix pairs  $(t \hat{W},\frac{1}{t}\hat{H})$ and increase $t$ until
 \[t \left (\sum\limits_{i,k}\hat{W}_{ik}\right) + {1\over t}\left(\sum\limits_{k,j}\hat{H}_{kj} \right)= C.\]
 This leads to  a pair of matrices that are FOSP of \eqref{main2}.

\section{Proof of Lemma~\ref{l:main1}}\label{s:proof_lemma_1}
Let $(W^*,H^*)$ be a SOSP of the problem  \eqref{main2} with  $c\ne~0$. We will arrive at a contradiction by showing
the existence of   a pair of matrices ${(W^s,H^s)}$ satisfying:

\begin{align}
&\frac{\partial F(W^*,H^*)}{\partial W_{ik}^*} > c \implies W^s_{ik}=0,  \label{eq:1} \\
 & \frac{\partial F(W^*,H^*)}{\partial H_{kj}^*} > c \implies H^s_{kj}=0,\label{eq:2} \\
& \sum_{i,k}W_{ik}^s + \sum_{k,j}H_{kj}^s =0,  \label{eq:3} \\
&  s^\top \nabla^2 F(W^*,H^*)s < 0, \label{eq:4}
\end{align}
where we set $s^{\top} =  (\textrm{vec}(W^s) ^{\top}, \textrm{vec}(H^s) ^{\top})$.
\noindent
We start with  a technical claim that will be used up next.

\begin{lemma}\label{l:sum}
Let $(W^*,H^*)$ be a FOSP for \eqref{main2},   for some constant $c \neq 0$. Then for all $k \in \{1,\ldots,r\}$, we have
\[\sum_{i=1}^n W_{ik}^* = \sum_{j=1}^m H_{kj}^*.\]
\end{lemma}
\begin{proof}Direct calculation reveals that:

\begin{align*}
c \sum_{i=1}^n W_{ik}^* &= \sum_{i: W_{ik}^* > 0}^n W_{ik}^* \cdot c\\
&= \sum_{i: W_{ik}^* > 0} W_{ik}^* \cdot \frac{\partial F(W^*,H^*)}{\partial W_{ik}}\\
&= \sum_{i=1}^n W_{ik}^* \left[-2 \sum_{j=1}^m \left (V_{ij} - \sum_{\ell = 1}^r W_{i \ell}^* H_{\ell j }^* \right)H_{kj}^*\right]\\
&=
\sum_{j=1}^m H_{kj}^* \left[-2 \sum_{i=1}^n\left(V_{ij} - \sum_{\ell = 1}^r W_{i \ell}^* H_{\ell j }^*\right)W_{ik}^*\right]\\
&=
\sum_{j:H_{kj}^*>0 }^m H_{kj}^* \left[-2 \sum_{i=1}^n \left(V_{ij} - \sum_{\ell = 1}^r W_{i \ell}^* H_{\ell j }^*\right)W_{ik}^*\right]\\
&=c \sum_{j=1}^mH_{kj}^*.
\end{align*}
\end{proof}
We are now ready to describe the construction of the matrices $(W^s,H^s)$ satisfying \eqref{eq:1}-\eqref{eq:4}.  Setting
$$k = \argmax_{ 1 \leq \ell \leq r} \left (\sum_{i=1}^nW^*_{i\ell} + \sum_{j=1}^m H^*_{\ell j} \right),$$
it follows immediately by Lemma \ref{l:sum}  that
\be\label{dscfvvf}
\norm{W^{*k}}_{1} = \norm{H^{*}_k}_{1} \geq C / 2r.
\ee
The matrices $W^s,H^s$ are defined as follows: $W_{ik}^s = W_{ik}^*$, $H_{kj}^s = -H_{kj}^*$, while $W_{i \ell }^s = H_{\ell j}^s = 0$, i.e., the $k$-th column of $W^s$ coincides with the $k$-th column of $W^*$ and all other entries are zero.   It is immediate from the definition that Conditions~\eqref{eq:1}-\eqref{eq:2}  are satisfied, since $\frac{\partial F(W^*,H^*)}{\partial W_{ki}} >c$
implies $W_{ik}^* = 0$ (by Condition~\eqref{def:KKT2}) and thus $W_{ik}^s = 0$, by the definition  of $W^s$  (and analogously for $H^*_{kj}$).
Moreover, Condition~\eqref{eq:3} is satisfied since
\be
\sum_{i=1}^n \sum_{k=1}^r W^s_{ik} +
\sum_{j=1}^m \sum_{k=1}^r H^s_{kj}=\sum_{i=1}^nW^*_{ik}-\sum_{j=1}^m H_{kj}^*=0,
\ee
where the last equality follows from  Lemma \ref{l:sum}.

It remains to verify Condition~\eqref{eq:4}. By Lemma~\ref{l:main_technical}, we have that

\begin{eqnarray*}
s^{\top} \nabla^2F(W^*,H^*) s
&=&
({\rm vec}(W^s) ^{\top}, {\rm vec}(H^s) ^{\top}) \cdot \nabla^2 F(W^\ast,H^\ast)
\cdot ({\rm vec}(W^s), {\rm vec}(H^s))
\\
&=&
({\rm vec}(W^s) ^{\top}, {\rm vec}(H^s) ^{\top}) \cdot \begin{pmatrix}0 \\ -\nabla_{W^k}F(W^\ast,H^\ast) \\ 0 \\ \nabla_{H_k}F(W^\ast,H^\ast) \\ 0 \end{pmatrix}
\\
&=&\sum_{i=1}^n \left(-\frac{\partial F(W^\ast,H^\ast)}{\partial W_{ik}}\right) W^s_{ik} + \sum_{i=1}^n \frac{\partial F(W^\ast,H^\ast)}{\partial H_{kj}} H^s_{kj}\\
&=& \sum_{i=1}^n \left(-\frac{\partial F(W^\ast,H^\ast)}{\partial W_{ik}}\right) W^\ast_{ik} + \sum_{i=1}^n \frac{\partial F(W^\ast,H^\ast)}{\partial H_{kj}} (-H^\ast_{kj})\\
&=& 2\sum_{i=1}^n\sum_{j=1}^m V_{ij}W_{ik}^*H_{kj}^*
-2\sum_{i=1}^n\sum_{j=1}^m W_{ik}^*H_{kj}^*\sum_{\ell=1}^r W_{i \ell}^* H_{\ell j}^*\\
&\leq&
2\sum_{i=1}^n\sum_{j=1}^m V_{ij}W_{ik}^*H_{kj}^* -2\sum_{i=1}^n\sum_{j=1}^m W_{ik}^{*2}H_{kj}^{*2}\\
&\leq&
2\norm{V}_F \norm{W^{*k}  H_k^*}_F -2\norm{ W^{*k}  H_k^*}_{F}^2
\end{eqnarray*}
where  the first inequality follows from the fact that all the entries $W^*_{i \ell},H^*_{\ell j}$ are positive and the second from the Cauchy-Schwarz inequality. In order to satisfy Condition~\eqref{eq:4}, it remains  to show that $\norm{V}_F - \norm{W^{*k} H_k^*}_F < 0$. Indeed,
\begin{eqnarray*}
\norm{W^{*k} H_k^*}_F^2 &=& \sum_{i=1}^n\sum_{j=1}^m W_{ik}^{*2}H_{kj}^{*2}\\
&\geq& \frac{1}{nm} \left(\sum_{i=1}^n\sum_{j=1}^m W_{ik}^* H_{kj}^*\right)^2\\
&=& \frac{1}{nm} \left(\sum_{i=1}^n W_{ik}^* \right)^2 \left(\sum_{j=1}^m  H_{kj}^*\right)^2\\
&\geq& \frac{1}{16 nm} \left(\norm{W^{*k}}_{1}\right)^4,\\
\end{eqnarray*}
where the last inequality follows  by Claim~\ref{l:sum}.
Combining the above with \eqref{dscfvvf}
we arrive at,
\[\norm{W^{*k} H_k^*}_F \geq \frac{C^2}{4r^2\sqrt{nm}} > \norm{V}_F,\] where we used that  $C > 2r(nm)^{\frac{1}{4}}\sqrt{\norm{V}_F}$.

\section{Proof of Theorem~\ref{t:main2}}\label{s:proof_main_theorem}
Let $(W^*,H^*)$ be a second-order stationary point   of the problem \eqref{main2} which is not  a second-order stationary point of the unconstrained problem \eqref{main}.
Since  $(W^*,H^*)$ is a SOSP    of  \eqref{main2},  Lemma \ref{l:main1} implies  that $c=0$.

As in  the proof of Lemma~\ref{l:main1}, we will arrive at a contradiction by  identifying  a non-zero pair of matrices $(W^s,H^s)$ that satisfy the following $4$ conditions:

\begin{align}
 & \frac{\partial F(W^*,H^*)}{\partial W_{ik}}>0 \implies W^s_{ik}=0,\label{eqq:1} \\
& \frac{\partial F(W^*,H^*)}{\partial H_{kj}} > 0 \implies H^s_{kj}=0,\label{eqq:2}\\
& \sum_{i,k} W_{ik}^s + \sum_{k,j}H_{kj}^s =0, \label{eqq:3} \\
& s^{\top} \nabla^2 F(W^*,H^*)s < 0, \label{eqq:4}
\end{align}
where we set $s^{\top} =  (\textrm{vec}(W^s) ^{\top}, \textrm{vec}(H^s) ^{\top}).$

As  $(W^*,H^*)$  is not a SOSP of \eqref{main},
there exist a pair of matrices $(\hat{W},\hat{H})$ satisfying conditions \eqref{eqq:1}, \eqref{eqq:2} and~\eqref{eqq:4}. In the remainder of the proof, we use the pair $(\hat{W},\hat{H})$ to construct a new pair $(W^s,H^s)$ that also satisfies~\eqref{eqq:3}.
In the construction of the pair $(W^s,H^s)$, we will also use the pair $(\tilde{W},\tilde{H})$ of Lemma~\ref{l:main_technical} for $k = \argmax_{1 \leq \ell \leq r}\{\sum_{i,\ell}W^*_{i\ell} + \sum_{\ell,j}H^*_{\ell j}\}$. We remind that $\tilde{W}_{ik} = W^\ast_{ik},\tilde{H}_{kj} = -H^\ast_{kj}$ and $\tilde{W}_{i\ell} = \tilde{H}_{\ell j} = 0$. Moreover ($\vec{0}$ denotes the zero column vector of appropriate size),
\[
({\rm vec}(\tilde{W}) ^{\top}, {\rm vec}(\tilde{H}) ^{\top}) \cdot \nabla^2 F(W^\ast,H^\ast)=
\left(\begin{array}{c}
\vec{0}^{\top}\\ -\nabla_{W^k}F(W^\ast,H^\ast)\\ \vec{0}^{\top} \\\nabla_{H_k}F(W^\ast,H^\ast)\\ \vec{0}^{\top}
\end{array}\right).
\]

\noindent We are now ready to describe our construction. Consider the following family of matrix pairs:
\[(W^t,H^t) = (\tilde{W},\tilde{H}) + t (\hat{W},\hat{H})\footnote{slighlty abusing notation since, $W^t$ up to now means the $t$-th column},\]
which we show satisfies  Conditions~\eqref{eqq:1}, \eqref{eqq:2} and \eqref{eqq:4}.
Indeed, assuming that 
 $\frac{\partial F(W^*,H^*)}{\partial W_{ik}} > 0$, it follows  by \eqref{eqq:1} that  $\hat{W}_{ik}=0$, and furthermore,  as $(W^*,H^*)$  is a SOSP for \eqref{main2}, we also have that $W^*_{ik} =0$ and thus $\tilde{W}_{ik} =0$.
 Combining these we get   that  $W^t_{ik} = 0$, i.e., $(W^t,H^t)$ satisfies \eqref{eqq:1} for all~$t$.  Analogously,   if $\frac{\partial F(W^*,H^*)}{\partial H_{kj}} > 0$ it follows that  $H^t_{kj}=0.$

Lastly, we show that for  all $t\neq 0$, the pair of matrices $(W^t,H^t)$ satisfy \eqref{eqq:4}.
To simplify notation let,
\begin{itemize}
    \item $s^{\top}_t = (\textrm{vec}(W^t)^{\top}, \textrm{vec}(H^t) ^{\top})$, respectively $\hat{s}$ for $(\hat{W},\hat{H})$
    and $\tilde{s}$ for $(\tilde{W},\tilde{H})$.

    \item $\nabla^2 = \nabla^2F(W^*,H^*)$.

\end{itemize}
Then, we have that
\begin{eqnarray*}
s_t^{\top} \nabla^2 s_t &=& (\tilde{s} + t \hat{s})^\top \nabla^2 (\tilde{s} + t \hat{s})
= \tilde{s}^{\top} \nabla^2 \tilde{s} + 2 t\tilde{s}^{\top} \nabla^2 \hat{s} + t^2 \hat{s}^\top \nabla^2 \hat{s}\\
&=&
\underbrace{\left(0,-\nabla_{W^k}F(W^\ast,H^\ast),0,\nabla_{H_k}F(W^\ast,H^\ast),0 \right)^\top \cdot \tilde{s}}_{=0}\\
&+& \underbrace{2t (0,-\nabla_{W^k}F(W^\ast,H^\ast),0,\nabla_{H_k}F(W^\ast,H^\ast),0)^\top \cdot \hat{s}}_{=0}\\
&+& t^2 \underbrace{\hat{s}^\top \nabla^2 \hat{s}}_{<0}
\end{eqnarray*}
where for the second equality we use  Lemma~\ref{l:main_technical}.
The first term is zero since if $\tilde{W}_{ik} \neq 0$ implies that $W^\ast_{ik} \neq 0$ which implies that $\frac{\partial F(W^\ast,H^\ast)}{\partial W_{ik}}=0$ since $(W^\ast,H^\ast)$ is a SOSP (analogously for $H_{kj}$). By Conditions~\eqref{eqq:1}-\eqref{eqq:2}, we know that if $\frac{\partial F(W^\ast,H^\ast)}{\partial W_{ik}}>0$ then $\hat{W}_{ik}=0$ (analogously for $H_{kj}$), meaning that the second term is also zero. Finally, the third term is strictly negative as $(\hat{W},\hat{H})$ satisfies \eqref{eqq:4}.

 We complete the proof  of Theorem~\ref{t:main2}  by noting that for
 \[ t = - \frac{\sum_{i=1}^n\sum_{k=1}^r \tilde{W}_{ik} + \sum_{k=1}^r\sum_{j=1}^m \tilde{H}_{kj}}{\sum_{i=1}^n\sum_{k=1}^r \hat{W}_{ik} + \sum_{k=1}^r\sum_{j=1}^m \hat{H}_{kj}}
 = -\frac{\sum_{i=1}^n W_{i k}^*  - \sum_{j = 1}^m H_{kj}^*}{\sum_{i=1}^n\sum_{k=1}^r \hat{W}_{ik} + \sum_{k=1}^r\sum_{j=1}^m \hat{H}_{kj}},
 \]the remaining  Condition~\eqref{eqq:3}  is satisfied. An important detail is that $t \neq 0$ since $\sum_{i=1}^n W_{i k}^*  = \sum_{j = 1}^m H_{kj}^*$ contradicts with the assumption that $(W^*,H^*)$ is SOSP (recall the proof of Lemma~\ref{l:main1} in Section~\ref{s:proof_lemma_1}).
 \section{Examples and Experiments}\label{s:experiments}
In this section we present experimental evaluations on the quality of the solutions produced by the Multiplicative Weight Update
algorithm (Algorithm~\ref{alg:MWU}), which indicate convergence to the global minimizers. More precisely, for several
values of the parameters $n$ and $r$, we generated random $n \times n$ matrices with entries in $[0,1]$
and rank $r$ and we checked the quality of the solutions $W_{n \times r},H_{n \times r}$ produced by MWU.
This was done so as to ensure that the global minimum of the respective NMF problem is $0$, which served
as a benchmark on the quality of the solutions produced by MWU.
For all the conducted experiments MWU was able to find a solution with value arbitrarily
close to $0$ meaning that it always converged to the right factorization.
Figure~\ref{fig:boat1} illustrates the number of iterations needed MWU to converge to
solutions with error smaller than $1\%$ of the initial error, for various values of $n,r$.
\begin{center}
\begin{figure}[h!]
\centering
  \includegraphics[width=300pt]{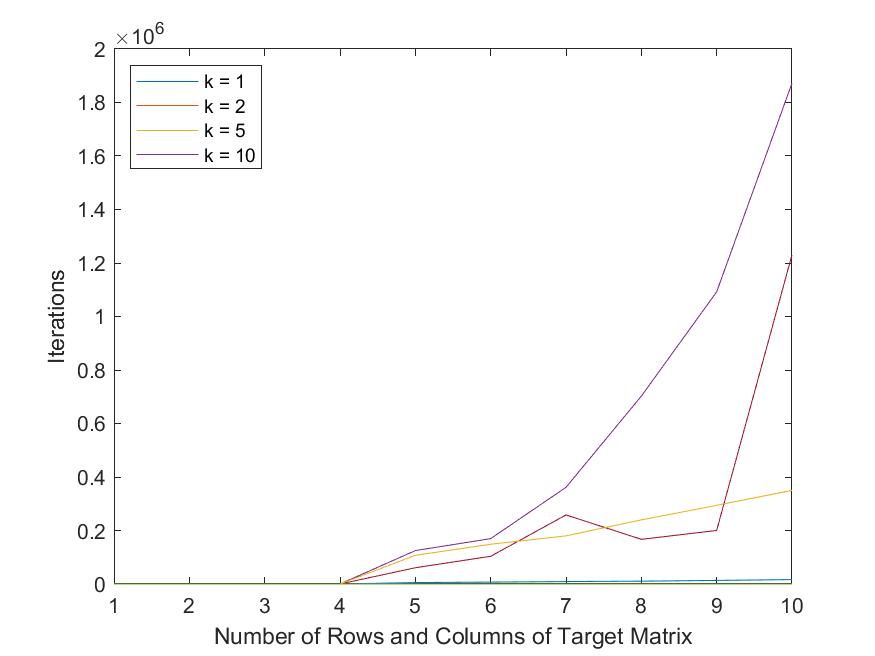}
  \caption{The figure depicts the number of iterations MWU needs to produce a solution with error smaller than $1\%$
  of the initial error. $V$ is a random matrix with rank $r$ and entries in $[0,1]$.}
  \label{fig:boat1}
\end{figure}
\end{center}
An important differentiation of the Multiplicative Weight Update depicted in Algorithm~\ref{alg:MWU} with
the previous multiplicative weight update and
gradient based approaches, is its concurrent way of updating the entries
of $W$ and $H$. Both the matrices $W^{t+1}$ and $H^{t+1}$ are updated
by using only the values of $W^{t}$ and $H^{t}$. We remark that that concurrency
in the updating step is very desirable,
since it permits more efficient implementations in parallel computing environments.
To the best of our knowledge, Algorithm~\ref{alg:MWU} is the first iterative method
for non-negative matrix factorization that converges while performing its updates in
a concurrent way. In all the previous gradient based approaches
(both in multiplicative weight algorithms and in alternating least squares)
the convergence properties heavily rely on the fact that
the entries of $W$ and $H$ are updating in an alternating way,
while their concurrent counterparts may fail to converge.
In Example~\ref{e:1} we present such a case for the Lee-Seung algorithm
in which the original version of the algorithm converges, while the concurrent
version fails to converge.
\begin{example}\label{e:1}
Consider the matrices
$V=
\begin{pmatrix}
1 & 0\\
0 & 1
\end{pmatrix}$
and $W^0= H^0 =
\begin{pmatrix}
1 & 1\\
1 & 1
\end{pmatrix}$ initialize the
concurrent version of Lee-Seung algorithm i.e.
$W_{ik}^{t+1} = W_{ik}^t \frac{(W_t^{T} V)_{ik}}{
(W_t^{T}W_t H_t)_{ik}}$ and
$H_{kj}^{t+1} = H_{kj}^t \frac{(VH^\top_t)_{kj}}{
(W_t H_t H_t^\top)_{kj}}$.
Then $(W^t,H^t)$ oscillates.
\end{example}
\begin{center}
\begin{figure}[h!]
\centering
  \includegraphics[width=300pt]{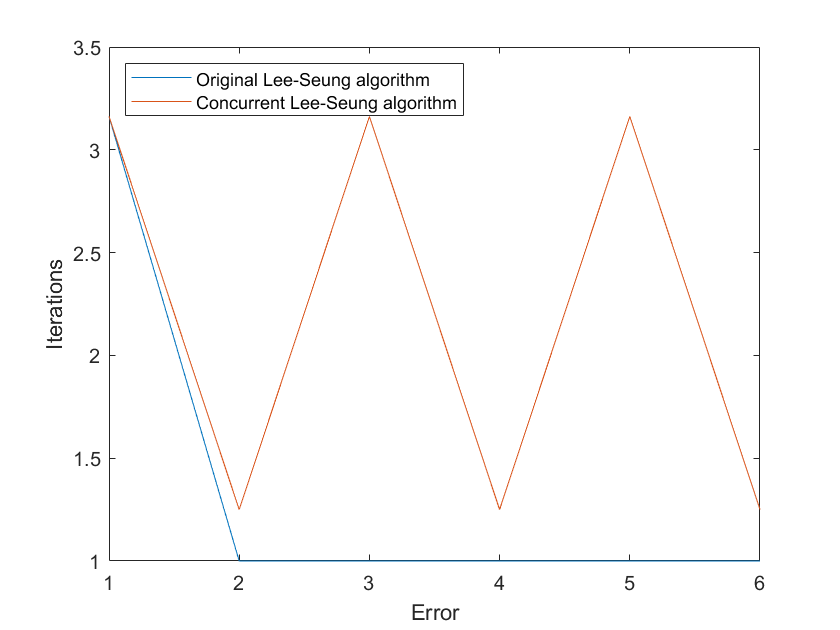}
  \caption{The error of the concurrent Lee-Seung algorithm versus the error of the original algorithm in the NMF instance described in Example~\ref{e:1}.}
\end{figure}
\end{center}

\section*{Acknowledgements} Ioannis Panageas and Xiao Wang were supported by SRG ISTD 2018 136, NRF-NRFFAI1-2019-0003 and NRF2019NRF-ANR2019. Stratis Skoulakis was supported by NRF 2018 Fellowship NRF-NRFF2018-07. Antonios Varvitsiotis was supported by SRG ESD 2020 154.
\bibliographystyle{plain}
\bibliography{NMF}

\end{document}